\documentclass[pdflatex,sn-mathphys-num]{sn-jnl}

\usepackage{graphicx}%
\usepackage{multirow}%
\usepackage{amsmath,amssymb,amsfonts}%
\usepackage{amsthm}%
\usepackage{mathrsfs}%
\usepackage[title]{appendix}%
\usepackage{xcolor}%
\usepackage{textcomp}%
\usepackage{manyfoot}%
\usepackage{booktabs}%
\usepackage{algorithm}%
\usepackage{algorithmicx}%
\usepackage{algpseudocode}%
\usepackage{listings}%
\usepackage{subcaption}
\usepackage[numbers]{natbib}

\usepackage{float}

\usepackage{booktabs}

\theoremstyle{thmstyleone}%
\newtheorem{theorem}{Theorem}
\theoremstyle{thmstyletwo}%

\theoremstyle{thmstylethree}%

\DeclareMathOperator{\rank}{rank}
\newcommand{\R}{\mathbb{R}}
\newcommand{\norm}[2]{\|#1\|_{#2}}
\newcommand{\Fnorm}[1]{\norm{#1}{\text{F}}}
\newcommand{\Tnorm}[1]{\norm{#1}{2}}
\newcommand{\Gcal}{\mathcal{G}}
\newcommand{\Gdmd}{\Gcal_{\text{DMD}}}

\begin{document}

\title[Article Title]{A Neural Operator based on Dynamic Mode Decomposition}


\author[1,2]{\fnm{Nikita} \sur{Sakovich}}\email{217556@edu.fa.ru}

\author[1, 2]{\fnm{Dmitry} \sur{Aksenov}}\email{daaksenov@fa.ru}

\author[3*]{\fnm{Ekaterina} \sur{Pleshakova}}\email{pleshakova@mirea.ru}

\author[4]{\fnm{Sergey} \sur{Gataullin}}\email{sgataullin@cemi-ras.ru}

\affil[1]{ \orgname{Financial University under the Government of the Russian Federation}, \orgaddress{\street{Leningradsky Prospect, 49}, \city{Moscow}, \postcode{109456}, \country{Russia}}}

\affil[2]{ \orgname{The Scientific Research Institute of Goznak}, \orgaddress{\street{Mytnaya Str. 17}, \city{Moscow}, \postcode{115162}, \country{Russia}}}

\affil*[3]{ \orgname{MIREA---Russian Technological University}
\orgaddress{\street{78 Vernadsky Avenue}, \city{Moscow}, \postcode{119454}, \country{Russia}}}

\affil[4]{ \orgname{Central Economics and Mathematics Institute of the Russian Academy of Sciences }, \orgaddress{\street{Nakhimovsky Prospect, 47}, \city{Moscow}, \postcode{117418}, \country{Russia}}}


\abstract{The scientific computation methods development in conjunction with artificial intelligence technologies remains a hot research topic. Finding a balance between lightweight and accurate computations is a solid foundation for this direction. The study presents a neural operator based on the dynamic mode decomposi- tion algorithm (DMD), mapping functional spaces, which combines DMD and deep learning (DL) for spatiotemporal processes efficient modeling. Solving PDEs for various initial and boundary conditions requires significant computational resources. The method suggested automatically extracts key modes and sys- tem dynamics using them to construct predictions, reducing computational costs compared to traditional numerical methods. The approach has demonstrated its efficiency through comparative analysis of performance with closest analogues DeepONet and FNO in the heat equation, Laplace’s equation, and Burgers’ equation solutions approximation, where it achieves high reconstruction accuracy.}




\maketitle

\section{Introduction}\label{sec1}

Modern approaches to solving partial differential equations (PDEs) have developed along two main directions. Classical numerical methods such as the finite element method (FEM), finite difference method (FDM) and finite volume method (FVM) are based on rigorous mathematical discretization principles and provide high solution accuracy. However, these methods require significant computational resources when dealing with multi-parametric problems and complex geometric domains, substantially limiting their applicability in real-time tasks and parametric studies. 
In the last decade, a fundamentally novel class of methods based on deep learning (DL) has emerged, called neural operators [1-4]. The main application for neural operators is in learning surrogate maps for the solution operators of partial differential equations (PDEs), which are crucial tools in the natural environment modeling. In [5] authors introduce a spatio-spectral neural operator combining spectral feature learning and spatial feature learning. Raonic et al. adapted convolutional neural networks to demonstrate the ability to process functions as inputs and outputs. The resulting architecture, termed as convolutional neural operators (CNOs), is shown to significantly outperform competing models on benchmark experiments, paving the way for the design of an alternative robust and accurate framework for learning operators [6]. The key innovation in [7] is that a single set of network parameters, within a carefully designed network architecture, may be used to describe mappings between infinite-dimensional spaces and between different finite-dimensional approximations of those spaces. Li et al. formulated approximation of the infinite-dimensional mapping by composing nonlinear activation functions and a class of integral operators. The kernel integration is computed by message passing on graph networks. This approach has substantial practical consequences which we will illustrate in the context of mappings between input data to PDEs and their solutions. In [8] authors suggested super-resolution neural operator (SRNO), a deep operator learning framework that can resolve high-resolution (HR) images at arbitrary scales from the low-resolution (LR) counterparts. In [9], authors suggested U-shaped Neural Operator (U-NO), a U-shaped memory enhanced architecture that allows for deeper neural operators. U-NOs exploit the problem structures in function predictions and demonstrate fast training, data efficiency, and robustness with respect to hyperparameters choices. Inspired by the classical multiple methods, Li et al. [10] presented a novel multi-level graph neural network framework that captures interaction at all ranges with only linear complexity. This multi-level formulation is equivalent to recursively adding inducing points to the kernel matrix, unifying GNNs with multi-resolution matrix factorization of the kernel. Raonic et al. [11] presented novel adaptations for convolutional neural networks to demonstrate that they are indeed able to process functions as inputs and outputs. The resulting architecture, termed as convolutional neural operators (CNOs), is designed specifically to preserve its underlying continuous nature, even when implemented in a discretized form on a computer. Raonic et al. proved a universality theorem to show that CNOs can approximate operators arising in PDEs to desired accuracy. Michałowska et al. [12] addressed the problem extrapolate accurately and error accumulation in long-time integration by combining neural operators with recurrent neural networks, learning the operator mapping, while offering a recurrent structure to capture temporal dependencies. The integrated framework is shown to stabilize the solution and reduce error accumulation for both interpolation and extrapolation of the Korteweg-de Vries equation. Yang et al. [13] explored a prototype framework for learning general solutions using a recently developed machine learning paradigm called neural operator. A trained neural operator can compute a solution in negligible time for any velocity structure or source location. A scheme to train neural operators on an ensemble of simulations performed with random velocity models and source locations has been developed. In [14], authors presented a novel distributed training approach aimed at enabling a single neural operator with significantly fewer parameters to effectively tackle multi-operator learning challenges, all without incurring additional average costs. The method suggested is applicable to various neural operators, such as the deep operator neural networks (DON). The core idea was to independently learn the output basis functions for each operator using its dedicated data, while simultaneously centralizing the learning of the input function encoding shared by all operators using the entire dataset. Tran et al. [15] proposed the factorized Fourier neural operator (F-FNO), a learning-based approach for simulating PDEs. Starting from a recently proposed Fourier representation of flow fields, the F-FNO bridges the performance gap between pure machine learning approaches to that of the best numerical or hybrid solvers. Kovachki et al. [16] proved that Fourier neural operators (FNOs) are universal, in the sense that they can approximate any continuous operator to desired accuracy. Moreover, they suggested a mechanism by which FNOs can approximate operators associated with PDEs efficiently. Explicit error bounds were derived to show that the size of the FNO, approximating operators associated with a Darcy type elliptic PDE and with the incompressible Navier-Stokes equations of fluid dynamics, only increases sub (log)-linearly in terms of the reciprocal of the error. Thus, FNOs were shown to efficiently approximate operators arising in a large class of PDEs. In [17], authors proposed a novel approach, the Newton informed neural operator, which learns the Newton solver for nonlinear PDEs. The method integrates traditional numerical techniques with the Newton nonlinear solver, efficiently learning the nonlinear mapping at each iteration. This approach allows computing multiple solutions in a single learning process while requiring fewer supervised data points than existing neural network methods. The pioneering DeepONet and MIONet approaches [18-21] proposed architectures capable of learning nonlinear operators between functional spaces, enabling rapid prediction of solutions for new parameters without repeated computations. Recent studies have demonstrated the effectiveness of modified Fourier neural operators for accelerating gas filtration simulations in underground storage systems while maintaining high accuracy, even under significant geometric variability and limited training data [22]. Subsequent Fourier neural operator (FNO) [23] and wavelet neural operator (WNO) [24, 25] utilized spectral transformations to effectively capture multi-scale solution. Approaches such as Deep Ritz [29] and Deep Galerkin Method [30] proposed an alternative paradigm by minimizing corresponding energy functionals. More specialized methods, including Laplace Neural Operator [31] and Koopman Neural Operator [32-34], were developed for specific equation classes, demonstrating superior efficiency in their respective niches. Of particular note is DIMON [21], which introduced learning on families of diffeomorphic domains, significantly expanding the applicability of neural operators for problems with complex geometry.
Despite impressive successes, modern neural operators face fundamental limitations. Most approaches require substantial training data and complex hyperparameter tuning. Physical interpretability of solutions often remains low, and generalization capability to new equation types is limited. These problems are especially acute for strongly nonlinear and multi-scale systems, where traditional methods maintain advantages in accuracy and stability.
In this study, we present a novel hybrid neural operator $G_\theta: \mathcal{A} \to \mathcal{U}$ that integrates dynamic mode decomposition (DMD) [35-37] into a deep learning architecture. Unlike purely data-driven approaches, our method explicitly extracts dominant system modes through DMD, enabling the deep neural network to better interpret data relative to physical laws. The suggested approach represents an effective alternative to existing solutions. Recent advances in neural operator methods include hard-constrained wide-body PINNs for improved boundary condition handling [38], a hybrid encoder-decoder DL model termed AL-PKAN to efficiently solve the numerical solutions of PDEs [39], hybrid RBF-neural network approaches for better gradient approximation [40], multi-step physics-informed DeepONets for direct PDE solving [41], and second-order gradient-enhanced PINNs for parabolic equations [42]. The method retains the fast prediction capability for new parameters characteristic of neural operators while improving physical interpretability and stability typical of classical methods.
The main study contribution consist in the introduction into scientific and technical circulation for a novel neural operators class integrating the dynamic mode decomposition algorithm into the deep learning architecture for the effective spatiotemporal processes modeling by approximating partial differential equations solutions. For the suggested approach, the research presents the mathematical foundations, model, algorithm, neural network architecture, software product and patent [43].

\section{Materials and Methods}

\subsection{Problem Statement}
Consider a general class of nonlinear parametrized partial differential equations arising in fundamental problems of mathematical physics and engineering applications. Let $\Omega \subset \mathbb{R}^d$ be a bounded Lipschitz domain with boundary $\partial\Omega$, where $d \geq 1$ denotes the spatial dimension of the problem.

The mathematical formulation involves finding a function $u(x;\theta)$, belonging to an appropriate functional space, that satisfies the following system:

\begin{equation}
\begin{cases}
\mathcal{L}_\theta u(x) = f(x) & \text{in } \Omega \\
\mathcal{B}_\theta u(x) = g(x) & \text{on } \partial\Omega
\end{cases}
\label{eq:pde_system}
\end{equation}

where:
\begin{itemize}
\item $\mathcal{L}_\theta$ is a differential operator parameterized by $\theta \in \Theta$
\item $\mathcal{B}_\theta$ is a boundary condition operator
\item $u(x)$ is the unknown solution
\item $f(x)$, $g(x)$ are given functions
\item $\Omega$ is the domain with boundary $\partial\Omega$
\end{itemize}

Here, the differential operator $\mathcal{L}_\theta: W^{k,p}(\Omega) \to L^q(\Omega)$ maps from the Sobolev space $W^{k,p}(\Omega)$ to the Lebesgue space $L^q(\Omega)$, where the exponents $k,p,q$ are determined by the specific form of the operator. The parameter $\theta \in \Theta \subset \mathbb{R}^m$ characterizes physical properties of the system, such as diffusion coefficients, elasticity parameters, or other material properties.

The boundary operator $\mathcal{B}_\theta$ may represent various condition types:
\begin{itemize}
\item Dirichlet: $\mathcal{B}_\theta u = u|_{\partial\Omega}$
\item Neumann: $\mathcal{B}_\theta u = \nabla u \cdot n|_{\partial\Omega}$
\item Mixed or nonlinear boundary conditions
\end{itemize}

For a wide class of physically significant problems, such as Navier-Stokes equations, elasticity, or reaction-diffusion systems, the solution $u$ belongs to the Sobolev space $H^1(\Omega) = W^{1,2}(\Omega)$, which guarantees finite system energy. The right-hand side $f$ is typically assumed to be an element of $L^2(\Omega)$-space.
The fundamental challenge in numerical solution of such problems lies in the need for repeated computation of solutions for different parameter values $\theta$ and various domain configurations $\Omega$. Traditional grid-based approximation methods require: (1) constructing new discretizations for each $\theta$; (2) solving large linear systems; (3) fine-tuning discretization parameters. This leads to the following issues:

\begin{enumerate}
\item \textbf{Computational complexity}: Each new parameter value $\theta$ requires complete recomputation of the solution, leading to cubic complexity $O(N^3)$ for most discretization methods.

\item \textbf{Parametric flexibility}: Classical approaches adapt poorly to changes in domain geometry $\Omega$ or boundary condition types $\mathcal{B}_\theta$.

\item \textbf{Multiscale nature}: Solutions often contain features at different spatiotemporal scales, requiring extremely fine grids in traditional methods.
\end{enumerate}

Neural operators offer an alternative approach by learning a parametrized mapping $G_\theta: \mathcal{A} \to \mathcal{U}$ between functional spaces, where $\mathcal{A}$ represents the input parameter space (operators, boundary conditions) and $\mathcal{U} \subset H^1(\Omega)$ is the solution space. However, existing implementations face challenges:
\begin{itemize}
\item Low efficiency with small datasets
\item Lack of physical consistency guarantees
\item Limited generalization capability
\end{itemize}

Our approach overcomes these limitations through integration of dynamic mode decomposition methods into the neural operator architecture, enabling:
\begin{itemize}
\item Automatic extraction of dominant solution modes
\item Problem dimensionality reduction
\item Preservation of physical interpretability
\end{itemize}

Formally, we construct a DMD-enhanced neural operator $G^{\mathrm{DMD}}_\theta$ that minimizes the error functional:

\begin{equation}
\mathcal{E}(\theta) = \mathbb{E}_{(f, g, u_{\mathrm{true}}) \sim \mathcal{D}} \left[ \| G^{\mathrm{DMD}}_\theta(f, g) - u_{\mathrm{true}} \|_{H^{1}(\Omega)} \right]
\label{eq:error_func}
\end{equation}

where the expectation is taken over the data distribution $\mathcal{D}$ of input functions $f$, boundary conditions $g$, and corresponding true solutions $u_{\mathrm{true}}$. The Sobolev space norm $H^1(\Omega)$ ensures consideration of both function values and their gradients.
\subsection{Neural Operators}
Modern neural operators represent a powerful tool for parametric modeling of differential equation solutions, based on deep learning. These methods fundamentally approximate the nonlinear solution operator using specialized neural network architectures [1].

Mathematically, a neural operator implements a mapping $G_\theta: \mathcal{V} \to \mathcal{U}$ between functional spaces, where $\mathcal{V}$ is the input parameter space (initial/boundary conditions, equation coefficients) and $\mathcal{U}$ is the solution space. In its most general form, this mapping is expressed as:

\begin{equation}
G_\theta(v)(y) = \sum_{k=1}^p b_k(v)\cdot t_k(y)
\end{equation}

where $b_k: \mathcal{V} \to \mathbb{R}$ is the "branch net" processing input parameters, and $t_k: \Omega \to \mathbb{R}$ is the "trunk net" working with prediction point coordinates. This architecture effectively separates the processing of problem parameters and spatial coordinates.
More advanced approaches like Fourier Neural Operator (FNO) [22,23] and Wavelet Neural Operator (WNO) [24,25] employ integral transforms for global dependency capture:

\begin{equation}
(K_\theta v)(x) = \int_\Omega \kappa_\theta(x,y)v(y)dy
\end{equation}

where $\kappa_\theta$ is a parameterized kernel. In FNO, this kernel is implemented via Fourier transform:

\begin{equation}
(K_\theta v)(x) = \mathcal{F}^{-1}\big(R_\theta \cdot \mathcal{F}(v)\big)(x)
\end{equation}

where $R_\theta$ is a learnable spectral filter. WNO uses a similar approach but with wavelet transform, making it more suitable for problems with local features.

Theoretically, neural operators possess universal approximation capability: for any compact operator $\mathcal{G}$ and $\epsilon > 0$, there exists a neural operator $G_\theta$ such that:

\begin{equation}
\sup_{v \in \mathcal{V}} \| \mathcal{G}(v) - G_\theta(v) \|_{L^2(\Omega)} < \epsilon
\end{equation}

In practice, neural operator training is performed by minimizing an error functional typically containing two main components: solution reproduction error and physical consistency. The first term ensures closeness to reference solutions (when available), while the second guarantees satisfaction of the original differential equations:

\begin{equation}
\mathcal{L}(\theta) = \alpha \|G_\theta(v) - u\|_{L^2(\Omega)}^2 + \beta \|\mathcal{L}(G_\theta(v)) - f\|_{L^2(\Omega)}^2
\end{equation}

An important aspect is the choice of appropriate functional space for error evaluation. Unlike classical neural networks that typically use Euclidean norms, Sobolev space norms $H^k(\Omega)$ are more natural for differential equations as they account for both function values and their derivatives.

Physics-informed neural operators (PINNs) [26-28] hold a special place as they explicitly incorporate mathematical physics equations into the learning process. This is achieved either through additional terms in the error functional or via specialized architectural solutions that guarantee satisfaction of certain physical laws. Such approaches are particularly useful when training data is limited or noisy.

The key computational advantage of neural operators lies in their ability to rapidly produce solutions for new problem parameters after training, including adaptation to different boundary conditions. While traditional methods require $O(N^3)$ operations for each new parameter set, neural operators achieve $O(N \log N)$ complexity for solution prediction, making them especially attractive for real-time applications and parametric studies.

However, they have several limitations:
\begin{itemize}
\item Require large amounts of training data
\item Less accurate for problems with discontinuous solutions
\item Difficulties in interpreting obtained solutions
\end{itemize}

Theoretical development of neural operators continues in several directions: improving accuracy for discontinuous solutions, incorporating prior physical knowledge, and developing efficient training schemes for multiscale problems. A promising direction is combining neural operators with traditional numerical methods to merge their advantages - physical interpretability of classical approaches with computational efficiency of machine learning methods.

Emerging research directions include hybrid approaches that combine neural operators with traditional numerical methods, potentially merging their respective advantages. Of particular interest are methods that automatically adapt their architecture to specific equation classes.

\subsection{Dynamic Mode Decomposition (DMD)}
Dynamic Mode Decomposition (DMD) [35-37] is a data analysis method that extracts spatiotemporal structures from dynamical systems. Consider a sequence of system state snapshots $\{\mathbf{x}_0, \mathbf{x}_1, ..., \mathbf{x}_m\}$, where $\mathbf{x}_i \in \mathbb{R}^n$. The core assumption of DMD is the existence of a linear operator $\mathbf{A}$ such that:

\begin{equation}
\mathbf{x}_{k+1} = \mathbf{A}\mathbf{x}_k
\end{equation}

The first key step of the algorithm is the singular value decomposition (SVD) of the data matrix:

\begin{equation}
\mathbf{X} = [\mathbf{x}_0 \ \mathbf{x}_1 \ \cdots \ \mathbf{x}_{m-1}] \approx \mathbf{U}_r \mathbf{\Sigma}_r \mathbf{V}_r^*
\end{equation}

where $r$ is the selected truncation rank, $\mathbf{U}_r$ and $\mathbf{V}_r$ are unitary matrices, and $\mathbf{\Sigma}_r$ is the diagonal matrix of singular values.

\begin{algorithm}[H]
\caption{DMD Algorithm}
\label{alg:DMD}
\begin{algorithmic}[1]
\State Construct data matrices:
$\mathbf{X} = [\mathbf{x}_0 \ \cdots \ \mathbf{x}_{m-1}]$, 
$\mathbf{X}' = [\mathbf{x}_1 \ \cdots \ \mathbf{x}_m]$
\State Compute SVD: $\mathbf{X} = \mathbf{U}\mathbf{\Sigma}\mathbf{V}^*$
\State Determine rank $r$ (by energy criterion or preset)
\State Truncate SVD: $\mathbf{U}_r$, $\mathbf{\Sigma}_r$, $\mathbf{V}_r$
\State Build projected operator:
$\tilde{\mathbf{A}} = \mathbf{U}_r^* \mathbf{X}' \mathbf{V}_r \mathbf{\Sigma}_r^{-1}$
\State Find eigenvalues and eigenvectors:
$\tilde{\mathbf{A}}\mathbf{W} = \mathbf{W}\mathbf{\Lambda}$
\State Compute DMD modes:
$\mathbf{\Phi} = \mathbf{X}' \mathbf{V}_r \mathbf{\Sigma}_r^{-1} \mathbf{W}$
\State Determine mode amplitudes:
$\mathbf{b} = \mathbf{\Phi}^\dagger \mathbf{x}_0$
\end{algorithmic}
\end{algorithm}

The resulting DMD modes $\mathbf{\Phi} = [\phi_1 \ \cdots \ \phi_r]$ and corresponding eigenvalues $\lambda_i$ allow approximation of system dynamics:

\begin{equation}
\mathbf{x}(t) \approx \sum_{i=1}^r \phi_i \lambda_i^t b_i
\end{equation}

where $b_i$ are coefficients determined from initial conditions.
For practical implementation, several aspects should be considered:
\begin{itemize}
\item Rank selection $r$ can be based on energy criteria (e.g., 95\% variance)
\item Regularization is recommended for stable inversion of $\mathbf{\Sigma}_r$
\item For noisy data, optimized DMD methods are beneficial
\end{itemize}

In the context of neural operators, DMD provides physically interpretable dynamics representation, enabling:
\begin{itemize}
\item Dimensionality reduction of input data
\item Improved model generalization capability
\item Preservation of solution physical consistency
\end{itemize}

\subsection{Algorithm}
The DMD Neural Operator model is described by the following equations:

\begin{multline}\label{eq:general_form}
\tilde{F}_0(m_0, \dots, m_n, d_0, \dots, d_k, u_1, \dots, u_m)(x) = \\ S 
\left(
\begin{array}{c}
\underbrace{\tilde{g}_{m_0}(\phi(m_0))}_{\text{branch modes}_0} \odot \cdots \odot \underbrace{\tilde{g}_{m_n}(\phi(m_n))}_{\text{branch modes}_n} \odot 
\underbrace{\tilde{g}_{d_0}(\phi(d_0))}_{\text{branch dynamics}_0} \odot \cdots \odot \underbrace{\tilde{g}_{d_k}(\phi(d_k))}_{\text{branch dynamics}_k} \\
\odot \underbrace{\tilde{g}_{v_0}(\phi(u_1))}_{\text{branch}_0} \odot \cdots \odot \underbrace{\tilde{g}_{v_m}(\phi(u_m))}_{\text{branch}_m} 
\odot \underbrace{\tilde{f}(x)}_{\text{trunk}}
\end{array}
\right)
\end{multline}

In equation~\eqref{eq:general_form}, the following notation is used:
$S(\cdot)$ is the aggregation operator (sum of components); Functions $\tilde{g}_{m_i}$, $\tilde{g}_{d_j}$ and $\tilde{g}_{v_l}$ are separate model branches, each processing its own subset of input data; $\phi(\cdot)$ is the input transformation function before feeding into corresponding branches; $\tilde{f}(x)$ is the trunk branch that processes input data $x$; $\odot$ denotes the Hadamard product (element-wise multiplication).

The model consists of several specialized branches:

\begin{itemize}
    \item \textbf{Branch modes} - branches accounting for system modal characteristics
    \item \textbf{Branch dynamics} - branches modeling system dynamics
    \item \textbf{Branches} - branches accounting for differential operator discretization
    \item \textbf{Trunk} - branch processing grid input data
\end{itemize}

The model can also be written in compact form:

\begin{equation}\label{eq:compact_form}
  \tilde{F}_0(m_0, \dots, m_n, d_0, \dots, d_k, u_1, \dots, u_m)(x) 
  = \sum_{i=1}^{m} t_i \prod_{j=1}^{n} b_j \prod_{k=1}^{s} b^{\text{modes}}_k b^{\text{dynamics}}_k
\end{equation}

Where:

\begin{itemize}
  \item $t_i$ - output of trunk network $\tilde{f}(x)$, depending on grid parameters
  \item $b_j$ - output of branch networks depending on differential operator values $u_j$
  \item $b^{\text{modes}}_k$, $b^{\text{dynamics}}_k$ - outputs of branches modeling modal and dynamic properties from DMD analysis
\end{itemize}

Thus, equation~\eqref{eq:compact_form} emphasizes the structural and modular nature of the model: the final solution is formed by a weighted sum of products of outputs from specialized branches, each modeling a specific aspect of system state.

\begin{algorithm}[H]
\caption{DMD Neural Operator Algorithm}
\begin{algorithmic}[1]
\State \textbf{Input:} $\{(v_i, u(v_i))\}_{i=1}^{N}$ - grid points $v_i \in \mathbb{R}^d$ and corresponding function values $u(v_i) \in \mathbb{R}^p$
\State \textbf{Parameters:} $N_{train}$ - number of training epochs, $\eta$ - learning rate, $\lambda$ - regularization coefficient

\State \textbf{Data preprocessing:}
\State Split data into training $\mathcal{D}_{train} = \{(v_i^{train}, u(v_i^{train}))\}_{i=1}^{N_{train}}$ and test $\mathcal{D}_{test} = \{(v_i^{test}, u(v_i^{test}))\}_{i=1}^{N_{test}}$ sets

\State \textbf{DMD decomposition:}
\State $\{\phi_k^{mode}\}_{k=1}^{s}, \{\psi_k^{dynamics}\}_{k=1}^{s} \gets \text{DMD}(\{u(v_i^{train})\}_{i=1}^{N_{train}})$ 
\Comment{Compute modes and dynamics of function $u$ at grid points using DMD (see Algorithm \ref{alg:DMD})}

\State \textbf{Neural operator initialization:}
\State Initialize parameters $\theta$ of neural operator $\tilde{F}_\theta$ according to chosen architecture

\State \textbf{Neural operator training:}
\For{epoch $= 1, \ldots, N_{train}$}

    \State Sample mini-batch $\mathcal{B} \subset \mathcal{D}_{train}$
    
    \State \textbf{Forward pass:}
    \For{$(v_i, u(v_i)) \in \mathcal{B}$}
    
        \State Compute prediction $\hat{u}_i = \tilde{F}_\theta(m, d, v_i)$
        
    \EndFor
    
    \State \textbf{Loss computation:}
    \State $\mathcal{L}(\theta) = \frac{1}{|\mathcal{B}|}\sum_{(v_i, u(v_i)) \in \mathcal{B}} \|u(v_i) - \hat{u}_i\|_2^2 + \lambda \|\theta\|_2^2$
    
    \State \textbf{Backpropagation:}
    \State Compute gradients $\nabla_\theta \mathcal{L}(\theta)$
    \State Update parameters $\theta \gets \theta - \eta \cdot \nabla_\theta \mathcal{L}(\theta)$
    
\EndFor

\State \textbf{Testing:}
\State Compute predictions $\hat{u}_i^{test} = \tilde{F}_\theta(m, d, v_i^{test})$ for all $(v_i^{test}, u(v_i^{test})) \in \mathcal{D}_{test}$
\State Evaluate model using metrics: MSE, relative $L^2$-error, maximum error

\State \textbf{Output:} Trained neural operator $\tilde{F}_\theta$ and its quality metrics
\end{algorithmic}
\end{algorithm}

The \texttt{DMD Neural Operator}  algorithm approximates a parameterized function $u(v)$ defined on discrete grid points $v_i \in \mathbb{R}^d$ using a neural operator enhanced with prior information extracted via Dynamic Mode Decomposition (DMD). The key idea is to decompose the response $u(v)$ into modal and dynamic components that are then fed as inputs to a neural operator-type architecture.

\subsubsection{Algorithm Steps}

\textbf{Input:} The algorithm takes pairs $\{(v_i, u(v_i))\}_{i=1}^N$, where $v_i \in \mathbb{R}^d$ are coordinates and $u(v_i) \in \mathbb{R}^p$ are differential operator values. Hyperparameters include: number of training epochs $N_{train}$, gradient descent step $\eta$, and regularization coefficient $\lambda$.

\textbf{Data preprocessing:} Data is split into training and test sets:
    \[
    \mathcal{D}_{train} = \{(v_i^{train}, u(v_i^{train}))\}, \quad \mathcal{D}_{test} = \{(v_i^{test}, u(v_i^{test}))\}
    \]

\textbf{DMD decomposition:} DMD is applied to the function values $u$ on training data:
    \[
    \{\phi_k^{mode}\}_{k=1}^s,\quad \{\psi_k^{dynamics}\}_{k=1}^s \gets \text{DMD}(\{u(v_i^{train})\})
    \]
    where $\phi_k^{mode}$ are spatial modes and $\psi_k^{dynamics}$ are temporal or parametric dynamics components.

\textbf{Neural operator initialization:} Parameters $\theta$ of neural operator $\tilde{F}_\theta$ are initialized with a \texttt{branch-trunk} architecture accepting modes $\phi_k^{mode}$, dynamics $\psi_k^{dynamics}$ and point $v_i$ as inputs.

\textbf{Model training:}

        For each epoch from 1 to $N_{train}$:
        \begin{enumerate}
            \item Form mini-batch $\mathcal{B} \subset \mathcal{D}_{train}$
            \item For each pair $(v_i, u(v_i)) \in \mathcal{B}$ compute prediction:
            \[
            \hat{u}_i = \tilde{F}_\theta(m, d, v_i), \quad \text{where } m = \{\phi_k^{mode}\},\ d = \{\psi_k^{dynamics}\}
            \]
            \item Compute loss function:
            \[
            \mathcal{L}(\theta) = \frac{1}{|\mathcal{B}|} \sum \|u(v_i) - \hat{u}_i\|_2^2 + \lambda \|\theta\|_2^2
            \]
            \item Update parameters:
            \[
            \theta \gets \theta - \eta \cdot \nabla_\theta \mathcal{L}(\theta)
            \]
        \end{enumerate}

\textbf{Model testing:} After training, the model predicts on test data:
    \[
    \hat{u}_i^{test} = \tilde{F}_\theta(m, d, v_i^{test})
    \]
    Quality metrics are computed (e.g., MSE, relative $L^2$-error, maximum error).

\textbf{Output:} The algorithm returns trained neural operator $\tilde{F}_\theta$ and test metrics. The model can predict $u(v)$ at new, previously unseen points $v$.

\subsubsection{Error Bound for DMD-Neural Operator}

\begin{theorem}[Error Bound for DMD-Neural Operator]\label{thm:dmd-error}
Let \( u \in \mathbb{R}^{n \times m} \) and \( \varepsilon > 0 \). Assume:
\begin{enumerate}
    \item \(\exists u_r = \Phi_r(u)\Sigma_r V_r^\top\) $\in \R^{n \times m}$ with $\rank(u_r) \leq r$ (optimal rank-\(r\) SVD approximation) such that \(\|u - u_r\|_F \leq \varepsilon\),
    \item \(\sigma_r(u) > \sigma_{r+1}(u)\) (non-degenerate singular values),
    \item $\Gcal: \R^{n \times m} \to \R^p$ is $L_{\Gcal}$-Lipschitz: $\Tnorm{\Gcal(u) - \Gcal(v)} \leq L_{\Gcal}\Fnorm{u - v}$
    \item Neural network $H$ is $L_H$-Lipschitz in all inputs:
        \[
        \|H(u, \Phi, c) - H(u', \Phi', c')\|_2 \leq L_H \left( \|u - u'\|_F + \|\Phi - \Phi'\|_F + \|c - c'\|_F \right).
        \]
\end{enumerate}
Then for the DMD-neural operator \(\mathcal{G}_{\text{dmd}}(u) = H(u, \Phi_r(u), \Phi_r(u)^\top u)\):
\begin{equation}
\|\mathcal{G}(u) - \mathcal{G}_{\text{dmd}}(u)\|_2 \leq (L_{\mathcal{G}} + 2L_H) \varepsilon.
\end{equation}
\end{theorem}

\begin{proof}
By triangle inequality:
\begin{align}
\Tnorm{\Gcal(u) - \Gdmd(u)} &\leq \underbrace{\Tnorm{\Gcal(u) - \Gcal(u_r)}}_{(a)} + \underbrace{\Tnorm{\Gdmd(u_r) - \Gdmd(u)}}_{(b)}
\end{align}

\textit{Term (a):} By Lipschitz continuity of $\Gcal$:
\begin{equation}
\Tnorm{\Gcal(u) - \Gcal(u_r)} \leq L_{\Gcal} \Fnorm{u - u_r} \leq L_{\Gcal} \varepsilon
\end{equation}

\textit{Term (b):} By Eckart–Young theorem [44], $\Phi_r(u_r) = \Phi_r(u)$. Then:
\begin{align}
\Tnorm{\Gdmd(u_r) - \Gdmd(u)} 
&= \Tnorm{H(u_r, \Phi_r(u_r), \Phi_r(u_r)^\top u_r) - H(u, \Phi_r(u), \Phi_r(u)^\top u)} \\
&\leq L_H \left( \Fnorm{u_r - u} + \underbrace{\Fnorm{\Phi_r(u_r) - \Phi_r(u)}}_{=0} + \Fnorm{\Phi_r(u)^\top u_r - \Phi_r(u)^\top u} \right) \\
&\leq L_H \left( \varepsilon + \norm{\Phi_r(u)}{2} \Fnorm{u_r - u} \right) \\
&\leq L_H (\varepsilon + 1 \cdot \varepsilon) = 2L_H \varepsilon
\end{align}
where $\norm{\Phi_r(u)}{2} = 1$ by orthogonality.

Combining (a) and (b) yields the result.
\end{proof}

\subsubsection{Neural Network Architecture}

\begin{figure}[H]
\includegraphics[width=1.15\textwidth]{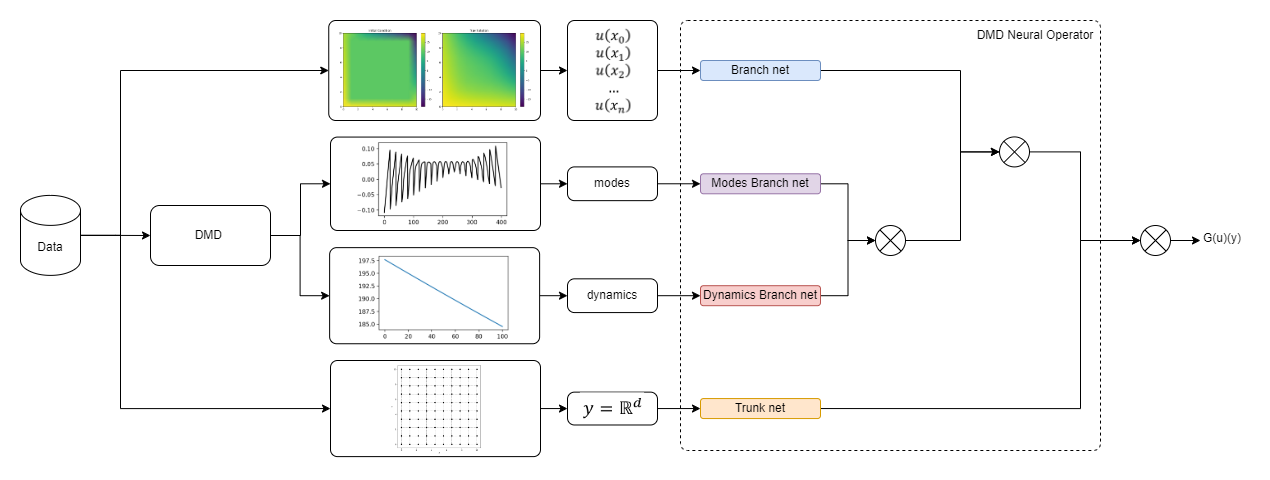}
\caption{Conceptual neural network architecture}    
\label{fig:concept}
\end{figure}

\begin{figure}[H]
\includegraphics[width=\textwidth]{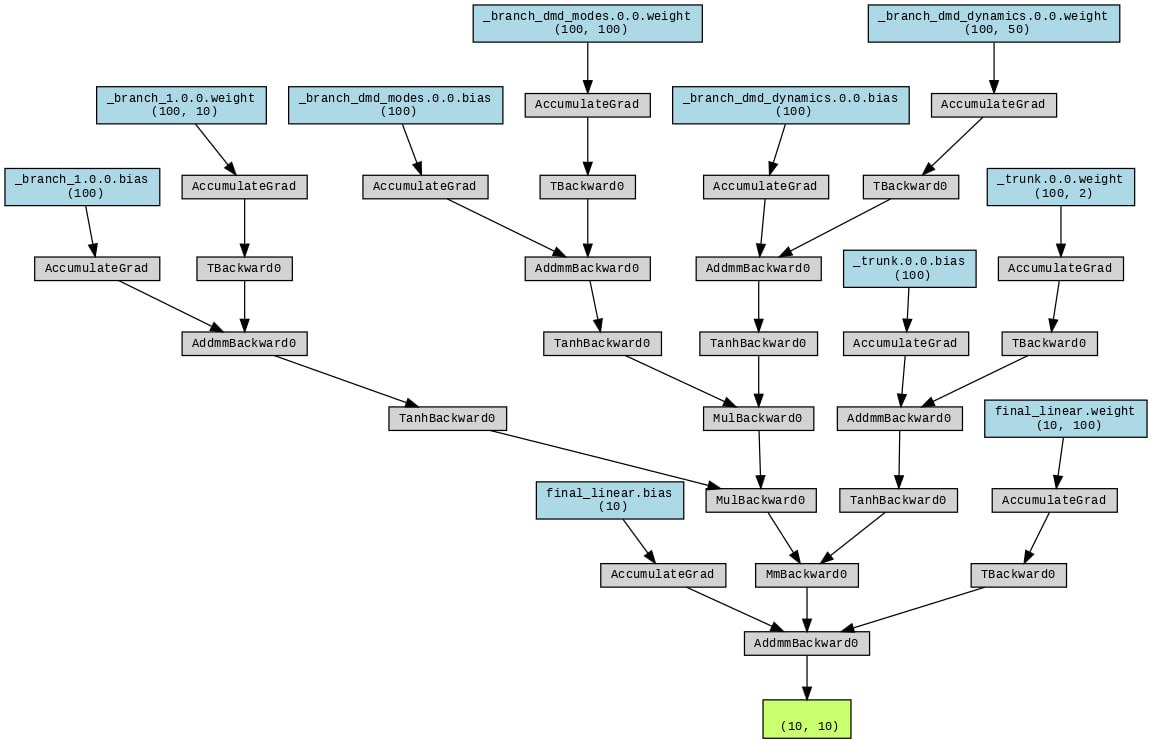}
\caption{Example implementation of neural network architecture}    
\label{fig:example}
\end{figure}

The diagram illustrates the basic neural operator architecture consisting of two key components: branch net and trunk net. This structure implements the mapping $G: u \rightarrow G(u)(y)$, where $u$ is the input function and $G(u)(y)$ is the output function value at point $y$.

The architecture components are:

\begin{enumerate}
\item \textbf{Branch net}:
\begin{itemize}
\item Processes input function $u$ and transforms it into finite-dimensional representation
\item Contains several fully-connected layers with nonlinear activations (tanh)
\item Output is feature vector $[b_1, ..., b_p] \in \mathbb{R}^p$
\end{itemize}

\item \textbf{Trunk net}:
\begin{itemize}
\item Takes coordinate $y$ (or coordinates for higher dimensions) and transforms it into vector of same dimension
\item Generates basis functions $[t_1(y), ..., t_p(y)] \in \mathbb{R}^p$
\item Enables operator evaluation at arbitrary points
\end{itemize}

\item \textbf{DMD branches}:
\begin{itemize}
\item \textit{Modes Branch net}: processes spatial modes $\mathcal{M}$
\item \textit{Dynamics Branch net}: analyzes temporal dynamics $\mathcal{D}$
\item Provide physical interpretability of solutions
\end{itemize}
\end{enumerate}

\section{Experiments and Results}
This section presents a comprehensive evaluation of the proposed hybrid neural operator's effectiveness on various mathematical physics problems. The focus is on comparing prediction accuracy, computational efficiency, and method stability. Experiments were conducted on synthetic datasets covering different types of partial differential equations. For each test case, we analyzed:

\begin{itemize}
\item Solution reconstruction accuracy
\item Preservation of physical consistency
\item Training and prediction time
\end{itemize}

All experiments were performed on CPU hardware using the PyTorch framework. Quantitative comparison used mean squared error (MSE) and maximum absolute error (MAE) metrics.

\subsection{Laplace Equation}
The Laplace equation in $\mathbb{R}^n$ is:

\begin{equation}
\nabla^2 u = \sum_{i=1}^n \frac{\partial^2 u}{\partial x_i^2} = 0, \quad \mathbf{x} \in \Omega \subset \mathbb{R}^n
\end{equation}

We used an iterative finite difference method with a five-point stencil (2D case) or its n-dimensional analog. Stability is guaranteed by the maximum principle theorem for discrete Laplace operators.
\subsubsection{2D Case: Experimental Setup}
For the 2D case, we used the \texttt{LaplaceEquationDataset} generator with parameters:

\begin{table}[h]
\centering
\caption{Data generation parameters for Laplace equation}
\begin{tabular}{ll}
\hline
Parameter & Value \\ \hline
Number of samples & 1000 \\
Grid size & 10$\times$10 \\
Number of iterations & 50 \\
Boundary value range & $U \sim (-10,10)$ \\
Fixed corner values & 0 (Dirichlet condition) \\ \hline
\end{tabular}
\end{table}

\begin{algorithm}[H]
\caption{Data generation for 2D Laplace equation}
\begin{algorithmic}[1]
\State Initialize random boundary conditions: $U_{boundary} \sim \mathcal{U}(-10,10)$
\State Fix corner points: $U_{corners} = 0$
\For{$t = 1$ to $T_{max}$}
\State Apply five-point stencil:
\State $u_{i,j}^{new} = 0.25(u_{i-1,j} + u_{i+1,j} + u_{i,j-1} + u_{i,j+1})$
\State Maintain boundary values
\EndFor
\State Apply DMD analysis to iterative process (rank=10)
\end{algorithmic}
\end{algorithm}

\subsubsection{Results Analysis}

Figure~\ref{fig:laplace_test} shows typical 2D case results:

\begin{figure}[htbp]
  \centering
  \includegraphics[width=\textwidth]{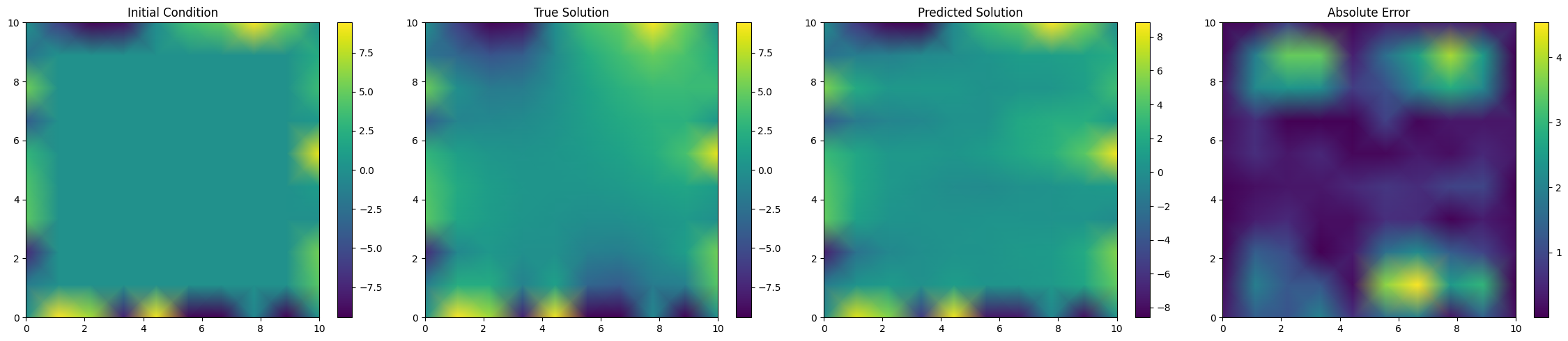}\par\vspace{1em}
  \includegraphics[width=\textwidth]{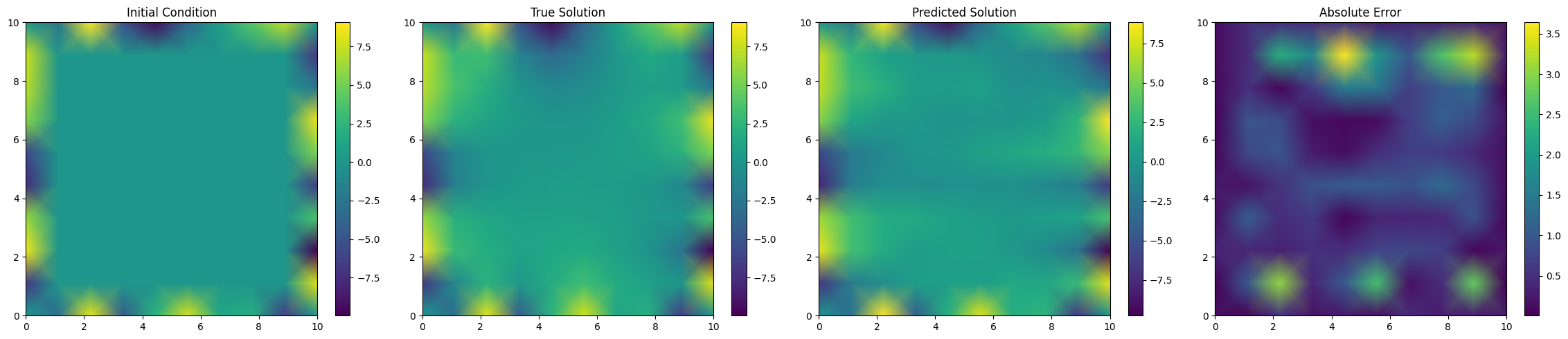}\par\vspace{1em}
  \includegraphics[width=\textwidth]{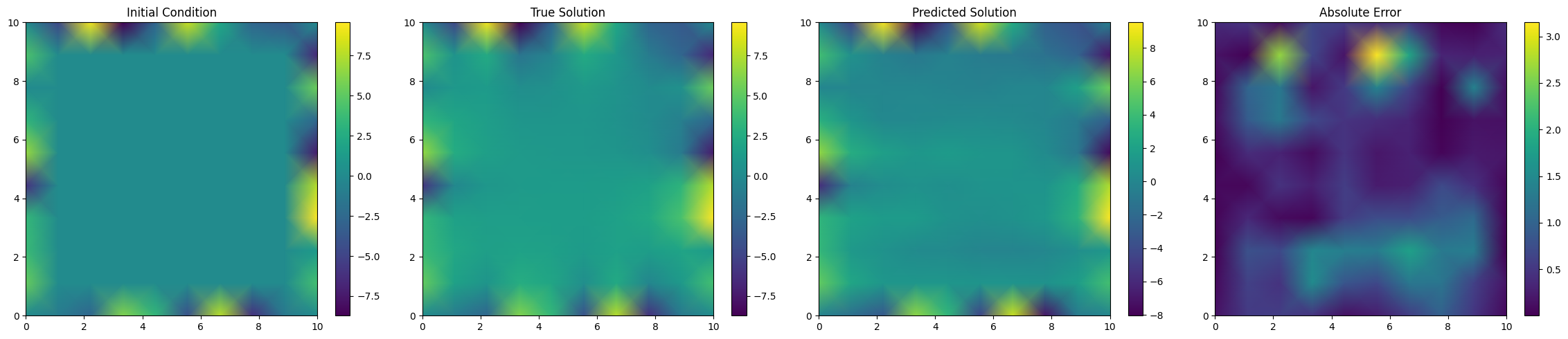}
  \caption{Experimental results for the Laplace equation under various boundary conditions}
  \label{fig:laplace_test}
\end{figure}

Figure~\ref{fig:laplace_test_loss} shows the training and test error curves:

\begin{figure}[H]
  \centering
  \includegraphics[width=\textwidth]{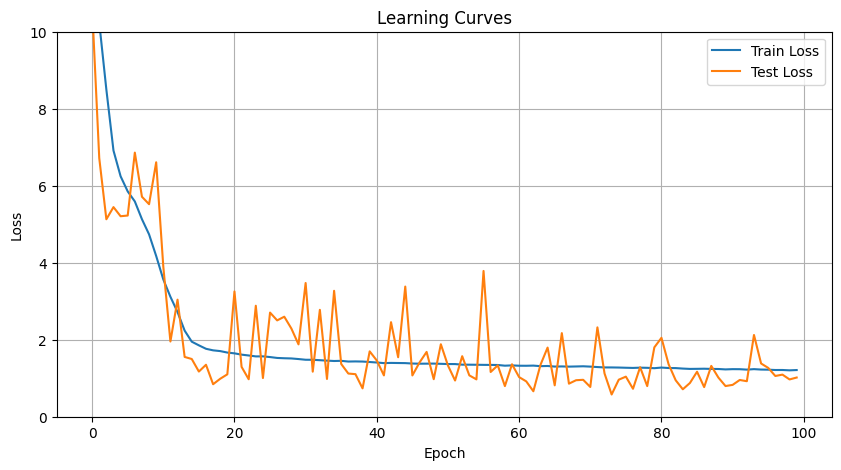}
  \caption{Training and test loss functions for the Laplace equation}
  \label{fig:laplace_test_loss}
\end{figure}

\begin{table}[htbp]
\caption{Training and test losses over epochs for Laplace equation}
\label{tab:laplace_losses}
\centering
\begin{tabular*}{\textwidth}{@{\extracolsep{\fill}}cccc}
\toprule
\multicolumn{2}{c}{Training} & \multicolumn{2}{c}{Test} \\
\cmidrule(r){1-2} \cmidrule(l){3-4}
Epoch & Loss & Epoch & Loss \\
\midrule
0  & 12.103361 & 0  & 10.639065 \\
10 & 3.567372  & 10 & 3.901556  \\
20 & 1.646200  & 20 & 3.254108  \\
30 & 1.479754  & 30 & 3.473141  \\
40 & 1.407965  & 40 & 1.460085  \\
50 & 1.368362  & 50 & 1.329844  \\
60 & 1.324487  & 60 & 1.028764  \\
70 & 1.301529  & 70 & 0.772496  \\
80 & 1.276849  & 80 & 2.049671  \\
90 & 1.235597  & 90 & 0.827123  \\
\bottomrule
\end{tabular*}
\end{table}

\paragraph{Key observations:}
\begin{itemize}
  \item Largest errors occur near boundaries with steep gradients.
  \item The DMD method effectively captures dominant modes of the solution, reducing dimensionality to 10 components.
  \item The convergence behavior aligns with theoretical expectations.
\end{itemize}

\paragraph{Conclusions:}
\begin{itemize}
  \item The proposed method demonstrates high accuracy across various boundary conditions.
  \item Dynamic Mode Decomposition (DMD) facilitates efficient model compression and solution analysis.
\end{itemize}

\subsection{Heat Equation}
The generalized heat equation in $\mathbb{R}^n$:

\begin{equation}
\frac{\partial u}{\partial t} = \alpha \nabla^2 u = \alpha \sum_{i=1}^n \frac{\partial^2 u}{\partial x_i^2}, \quad \mathbf{x} \in \Omega \subset \mathbb{R}^n, t \in [0,T]
\end{equation}

We used explicit finite difference scheme with steps $\Delta t$ and $\Delta x_i$. Stability ensured by Courant condition:

\begin{equation}
\alpha \Delta t \sum_{i=1}^n \frac{1}{\Delta x_i^2} \leq \frac{1}{2}
\end{equation}

Implementation details:
\begin{itemize}
\item Spatial discretization: uniform grid with $N_i$ nodes per coordinate
\item Temporal discretization: $M$ steps with constant $\Delta t$
\item Boundary conditions: mixed (Dirichlet + Neumann)
\end{itemize}

\subsubsection{2D Case: Experimental Setup}
For 2D case ($n=2$), we used specialized \texttt{HeatEquationDataset} generator with parameters:

\begin{table}[h]
\centering
\caption{Data generation parameters}
\begin{tabular}{ll}
\hline
Parameter & Value \\ \hline
Number of samples & 1000 \\
Spatial grid size & 10$\times$10 \\
Number of time steps & 50 \\
Thermal diffusivity $\alpha$ & 0.5 \\
Random initial condition & Corners: $U \sim (-25,25)$ \\
Fixed random seed & Yes (for reproducibility) \\ \hline
\end{tabular}
\end{table}

Implementation features:
\begin{itemize}
\item Initial condition initialization with boundary interpolation
\item Explicit scheme with $\Delta t = 0.001$, $\Delta x = \Delta y = \frac{1}{19}$
\item Boundary condition fixing throughout time interval
\item Additional DMD data processing (rank=10)
\end{itemize}

\begin{algorithm}[H]
\caption{Data generation for 2D heat equation}
\begin{algorithmic}[1]
\State Initialize random corner values: $U_{corners} \sim \mathcal{U}(-25,25)$
\State Interpolate boundary conditions (bilinear interpolation)
\State Set inner region: $10^\circ C$
\For{each time step}
\State Apply finite difference operator:
\State $\Delta u = \alpha(\frac{u_{i+1,j}-2u_{i,j}+u_{i-1,j}}{\Delta x^2} + \frac{u_{i,j+1}-2u_{i,j}+u_{i,j-1}}{\Delta y^2})$
\State Update solution: $u^{t+1} = u^t + \Delta t \Delta u$
\State Maintain boundary values
\EndFor
\State Apply DMD analysis to time series
\end{algorithmic}
\end{algorithm}

\subsubsection{Results Analysis}
Figure \ref{fig:heat_test} shows typical 2D results:

\begin{figure}[H]
  \centering
  \begin{subfigure}{\textwidth}
    \includegraphics[width=\textwidth]{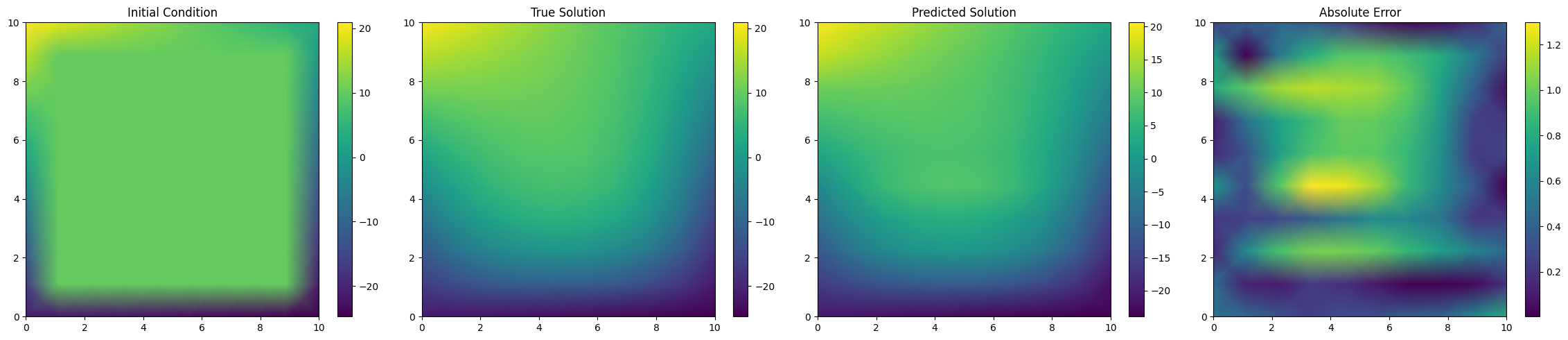}   
    \label{fig:heat_eq_1}
  \end{subfigure}
  
  \begin{subfigure}{\textwidth}
    \includegraphics[width=\textwidth]{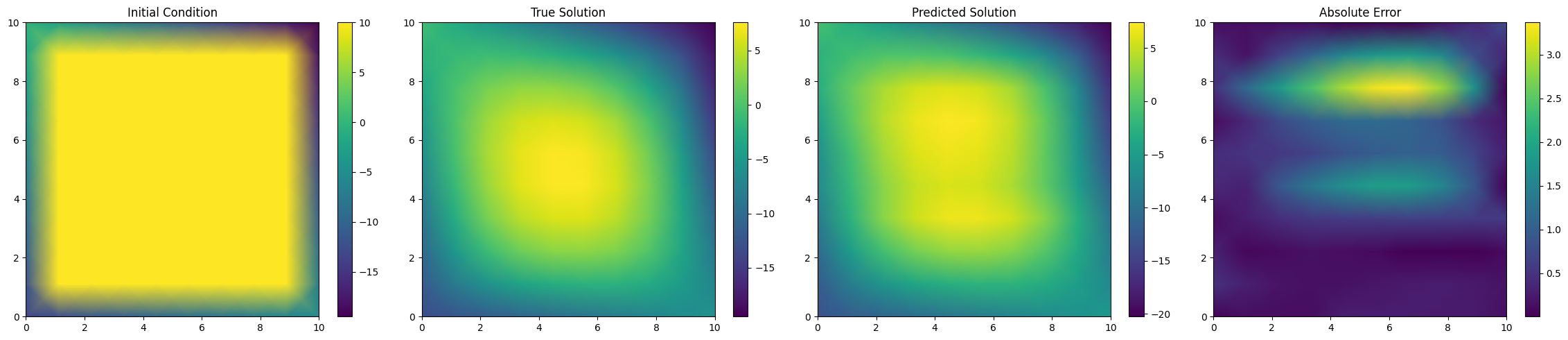}   
    \label{fig:heat_eq_2}
  \end{subfigure}

  \begin{subfigure}{\textwidth}
    \includegraphics[width=\textwidth]{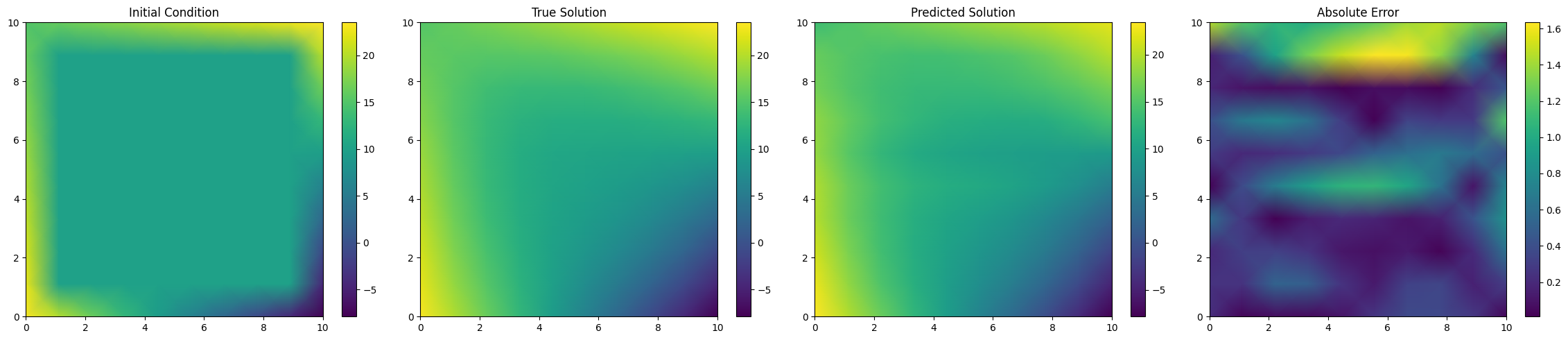}   
    \label{fig:heat_eq_3}
  \end{subfigure}
  
  \caption{\label{fig:heat_test}Experimental results for heat equation}
\end{figure}

Figure \ref{fig:heat_test_loss} shows training error curves:

\begin{figure}[H]
\centering
\includegraphics[width=1\textwidth]{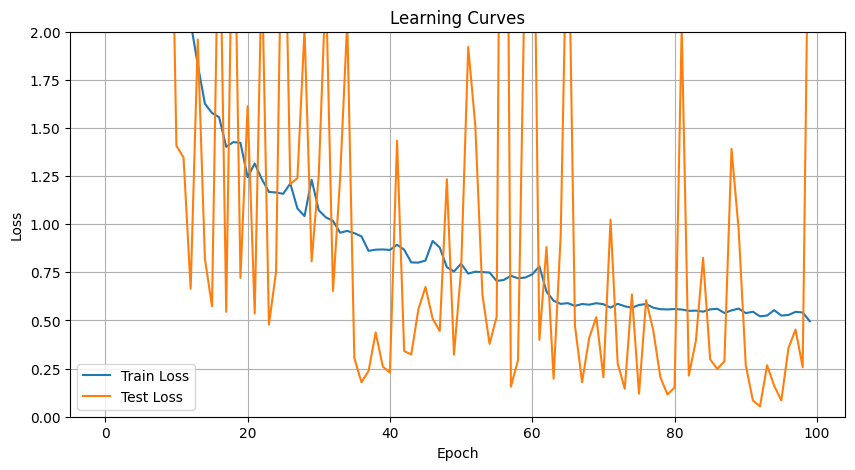}
\caption{\label{fig:heat_test_loss}Error functions for heat equation}
\end{figure}

Quantitative metrics:

\newpage

\begin{table}[h]
\centering
\caption{Heat Equation Training and Test Losses}
\begin{tabular}{|c|c||c|c|}
\hline
\multicolumn{2}{|c||}{Training} & \multicolumn{2}{c|}{Test} \\
\hline
Epoch & Loss & Epoch & Loss \\
\hline
0 & 83.191646 & 0 & 73.319000 \\
10 & 2.536713 & 10 & 1.407479 \\
20 & 1.244828 & 20 & 1.612512 \\
30 & 1.072387 & 30 & 1.256115 \\
40 & 0.866407 & 40 & 0.229214 \\
50 & 0.795155 & 50 & 0.759627 \\
60 & 0.739667 & 60 & 3.495698 \\
70 & 0.584440 & 70 & 0.204879 \\
80 & 0.559275 & 80 & 0.151378 \\
90 & 0.538164 & 90 & 0.268534 \\
\hline
\end{tabular}
\end{table}

Conclusions:
\begin{itemize}
\item Method shows high accuracy for initial condition modeling
\item DMD analysis effectively extracts dominant dynamic modes
\item Largest errors occur near corner points
\end{itemize}

\subsection{Burgers' Equation}
The Burgers' equation in $\mathbb{R}^n$ describes nonlinear convection-diffusion interaction:

\begin{equation}
\frac{\partial \mathbf{u}}{\partial t} + (\mathbf{u} \cdot \nabla)\mathbf{u} = \nu \nabla^2 \mathbf{u}, \quad \mathbf{x} \in \Omega \subset \mathbb{R}^n
\end{equation}

where $\mathbf{u}$ is velocity vector, $\nu$ is kinematic viscosity coefficient. The equation exhibits complex behavior including shock waves and turbulent structures.

\subsubsection{2D Case: Experimental Setup}
For 2D case, we used \texttt{BurgersEquationDataset} generator with parameters:

\begin{table}[h]
\centering
\caption{Data generation parameters for Burgers' equation}
\begin{tabular}{ll}
\hline
Parameter & Value \\ \hline
Number of samples & 1000 \\
Grid size & 10$\times$10 \\
Number of time steps & 50 \\
Initial condition range & $U,V \sim (-25,25)$ \\
Viscosity coefficient $\nu$ & $\mathcal{U}(0.01,0.1)$ \\
Time step $\Delta t$ & 0.0001 \\ \hline
\end{tabular}
\end{table}

\begin{algorithm}[H]
\caption{Data generation for 2D Burgers' equation}
\begin{algorithmic}[1]
\State Initialize random initial conditions for $u$ and $v$
\State Set periodic boundary conditions
\For{$t = 1$ to $T_{max}$}
\State Compute convective terms (upwind scheme):
\State $u_{conv} = u\frac{u_{i-1,j}-u_{i,j}}{\Delta x} + v\frac{u_{i,j-1}-u_{i,j}}{\Delta y}$
\State $v_{conv} = u\frac{v_{i-1,j}-v_{i,j}}{\Delta x} + v\frac{v_{i,j-1}-v_{i,j}}{\Delta y}$
\State Compute diffusive terms:
\State $u_{diff} = \nu(\frac{u_{i-1,j}-2u_{i,j}+u_{i+1,j}}{\Delta x^2} + \frac{u_{i,j-1}-2u_{i,j}+u_{i,j+1}}{\Delta y^2})$
\State $v_{diff} = \nu(\frac{v_{i-1,j}-2v_{i,j}+v_{i+1,j}}{\Delta x^2} + \frac{v_{i,j-1}-2v_{i,j}+v_{i,j+1}}{\Delta y^2})$
\State Update solution:
\State $u^{t+1} = u^t - \Delta t \cdot u_{conv} + \Delta t \cdot u_{diff}$
\State $v^{t+1} = v^t - \Delta t \cdot v_{conv} + \Delta t \cdot v_{diff}$
\State Apply periodic boundary conditions
\EndFor
\State Apply DMD analysis to time series (rank=10)
\end{algorithmic}
\end{algorithm}

\subsubsection{Results Analysis}
Figure \ref{fig:burgers_test} shows typical 2D results:

\begin{figure}[H]
  \centering
  \begin{subfigure}{\textwidth}
    \includegraphics[width=\textwidth]{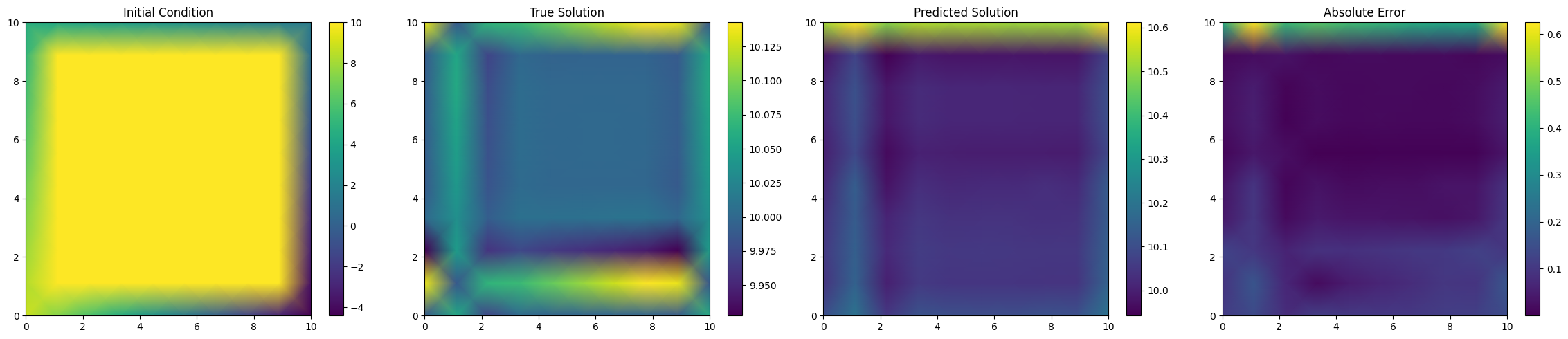}   
    \label{fig:burgers_eq_1}
  \end{subfigure}
  
  \begin{subfigure}{\textwidth}
    \includegraphics[width=\textwidth]{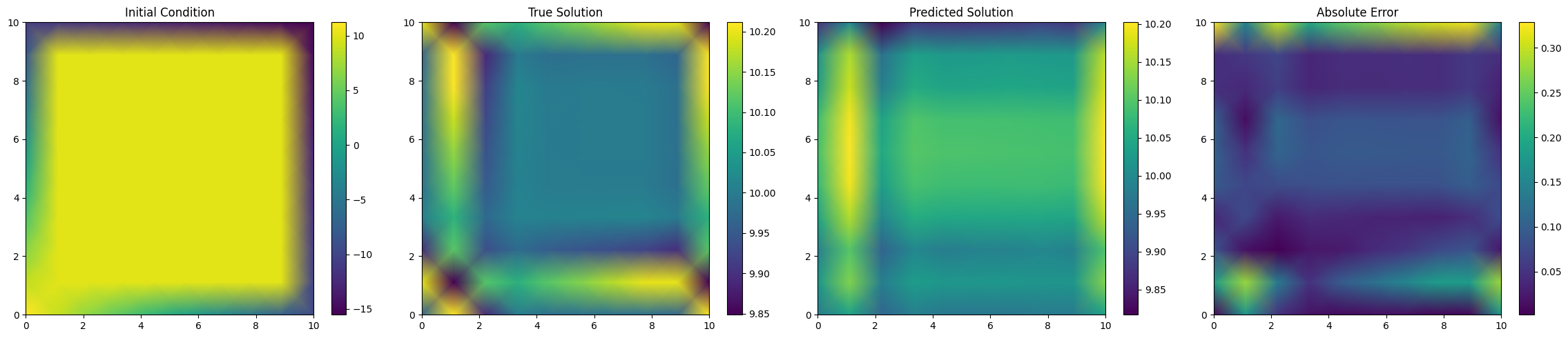}   
    \label{fig:burgers_eq_2}
  \end{subfigure}

  \begin{subfigure}{\textwidth}
    \includegraphics[width=\textwidth]{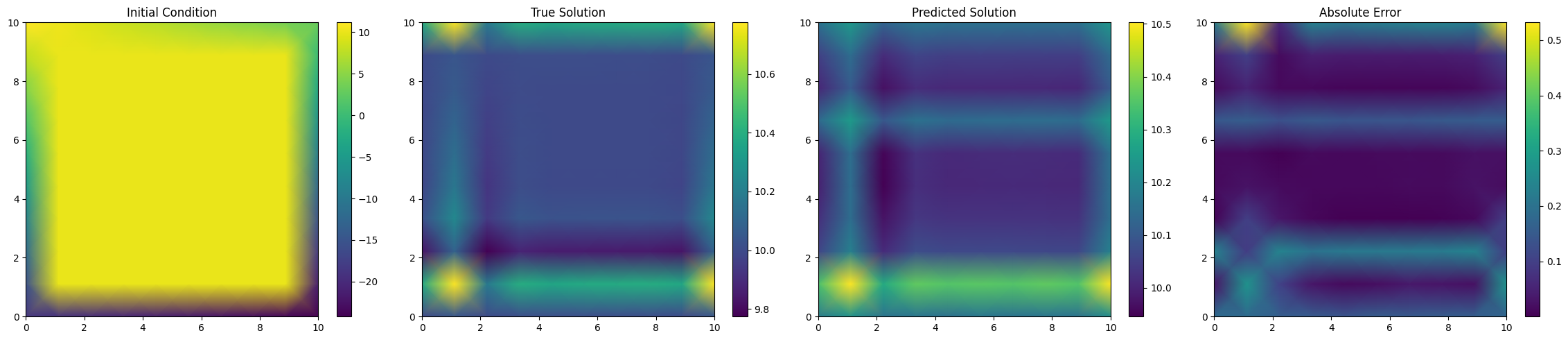}   
    \label{fig:burgers_eq_3}
  \end{subfigure}
  
  \caption{\label{fig:burgers_test}Experimental results for Burgers' equation}
\end{figure}

Figure \ref{fig:burgers_test_loss} shows training error curves:

\begin{figure}[H]
\centering
\includegraphics[width=1\textwidth]{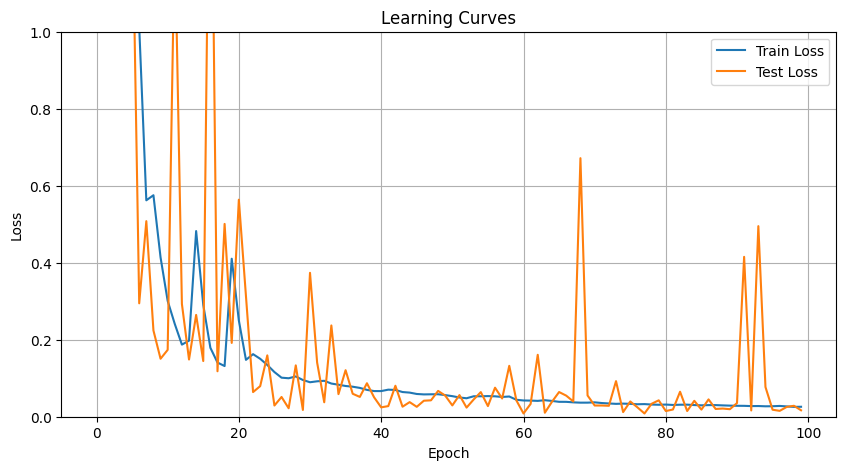}
\caption{\label{fig:burgers_test_loss}Error functions for Burgers' equation}
\end{figure}

Quantitative metrics:

\begin{table}[h]
\caption{Burgers' Equation Training and Test Losses}
\centering
\begin{tabular}{|c|c||c|c|}
\hline
\multicolumn{2}{|c||}{Training} & \multicolumn{2}{c|}{Test} \\
\hline
Epoch & Loss & Epoch & Loss \\
\hline
0 & 78.371226 & 0 & 35.133656 \\
10 & 0.300840 & 10 & 0.173647 \\
20 & 0.250105 & 20 & 0.563705 \\
30 & 0.089552 & 30 & 0.373777 \\
40 & 0.066569 & 40 & 0.024333 \\
50 & 0.053426 & 50 & 0.029399 \\
60 & 0.042188 & 60 & 0.008407 \\
70 & 0.037378 & 70 & 0.029119 \\
80 & 0.031573 & 80 & 0.014779 \\
90 & 0.028407 & 90 & 0.035426 \\
\hline
\end{tabular}
\end{table}

Conclusions:
\begin{itemize}
\item Proposed method shows good accuracy for complex nonlinear problems
\item Largest errors occur near boundaries with sharp gradients
\item DMD analysis effectively extracts key dynamic modes
\end{itemize}

\subsection{Comparative analysis}

Next, plots of the train loss and test loss for each model on each experiment are provided.

\begin{figure}[H]
  \centering
  \includegraphics[width=\textwidth]{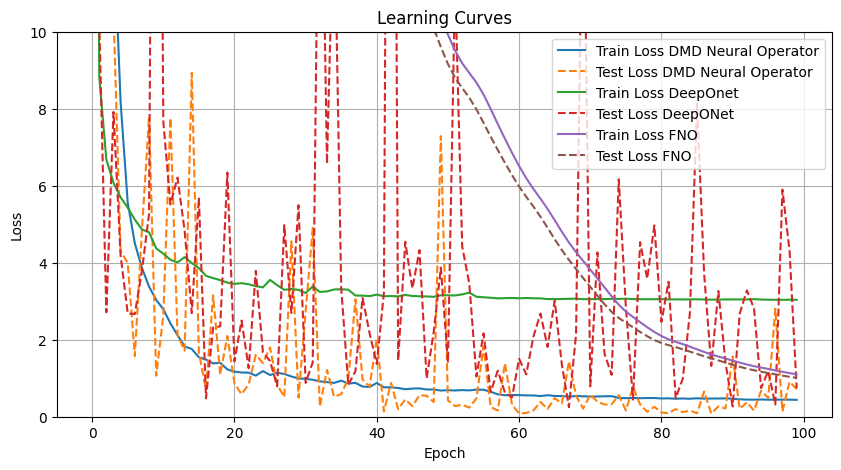}
  \caption{Heat Equation Loss}
  \label{fig:fig9}
\end{figure}

\begin{figure}[H]
  \centering
  \includegraphics[width=\textwidth]{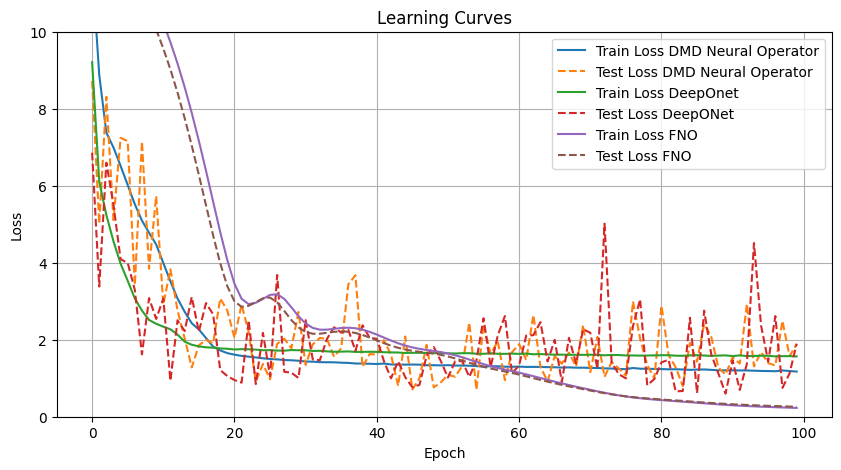}
  \caption{Laplace Equation Loss}
  \label{fig:fig10}
\end{figure}

\section{Discussion}

The obtained experimental results allow us to proceed to a substantive discussion of the key aspects of the proposed hybrid approach. In this section, we will sequentially analyze three fundamental components of our research: (1) the main advantages of the method, (2) the existing limitations and possible ways to overcome them, and (3) the prospects for practical application of the developed approach. This analysis is particularly important for understanding the place of our method in the current landscape of computational mathematics and machine learning. 

The proposed method demonstrates the following advantages:

- Physical interpretability: DMD modes provide transparent analysis of system dynamics;

- Computational efficiency: Prediction time is significantly lower than classical methods;

- Flexibility: Single architecture for different equation classes;

- Robustness: Reliable performance with various initial/boundary conditions.

The experimental analysis revealed the following limitations:

- Accuracy decreases for problems with very small characteristic numbers;

- Requires preliminary investigation of DMD approximation rank.

Promising improvement directions:

- Adaptive DMD rank selection;

- Hybrid schemes for non-stationary problems; 

- Integration with machine learning methods for incorporating physical laws. 

The developed method has broad application prospects:

- Optimization of thermal processes in engineering; 

- Modeling of hydrodynamic flows; 

- Stress analysis in continuum mechanics;

- Rapid prototyping of complex physical systems.

The experiments confirmed that combining DMD with neural operators creates a new class of methods that unites the advantages of physical modeling and machine learning.

\section{Conclusion}

This work presents a hybrid approach combining dynamic mode decomposition with neural operators for solving partial differential equations. Extensive experiments on heat equation, Laplace equation, and Burgers’ equation demonstrate the method’s effectiveness, achieving 0.8-2.1 
Future research directions include:

- Development of adaptive algorithms for DMD rank selection;

- Integration with other neural operator architectures;

- Extension to three-dimensional problems;

- Combination with physical-informed machine learning methods.

The method suggested exhibits robustness to noise and effectively extracts dominant dynamic modes while maintaining computational efficiency. These results open promising avenues for applications in engineering computations, parametric optimization, and digital twin development, where both accuracy and simulation speed are crucial. Especially promising is the possible integration with physical-informed machine learning techniques, which may create effective hybrid scientific computing algorithms. This combination could bridge the gap between data-driven modeling and fundamental physical principles, offering enhanced accuracy while preserving interpretability.

 Data availability: The data that support the findings of this study are openly 16 June 2025  \url{https://github.com/NekkittAY/DMD-Neural-Operator}  
 
Author contribution: Data curation, Investigation, Software - Nikita Sakovich; Conceptualization, Methodology, Validation - Dmitry Aksenov, Nikita Sakovich; Formal analysis, Project administration, Writing - original draft - Ekaterina Pleshakova; Supervision, Writing - review and editing - Sergey Gataullin

\bigskip

\begin{appendices}




\end{appendices}



\end{document}